\documentclass{article}

\usepackage{microtype}
\usepackage{graphicx}
\usepackage{subfigure}
\usepackage{booktabs} 
\usepackage{wrapfig}

\usepackage{algorithm}
\usepackage{algorithmic}

\usepackage{bbm}
\usepackage{bm}
\usepackage{amssymb}
\usepackage{amsthm}
\usepackage{mathrsfs}
\usepackage{color}
\usepackage{graphicx}
\usepackage{algorithm}
\usepackage{algorithmic}
\usepackage{subfigure}
\usepackage{epsfig}
\usepackage{epsf}
\usepackage{booktabs}
\usepackage{tabularx}
\usepackage{tabulary}
\usepackage{multirow}
\usepackage{graphics}
\usepackage{longtable}
\usepackage{subfigure}
\usepackage[table]{xcolor}    
\usepackage{hyperref}

\usepackage[accepted]{icml2019}

\usepackage{bm}
\usepackage{amsmath}
\usepackage{verbatim}
\usepackage{bm}

\usepackage{url}            

\usepackage{graphicx} 
\usepackage{subfigure}
\newcommand{\tabincell}[2]{\begin{tabular}{@{}#1@{}}#2\end{tabular}}

\newtheorem{thm}{Theorem}

\def\e{{\bf e}}

\def\w{{\bf w}}
\def\x{{\bf x}}
\def\y{{\bf y}}

\def\Q{{\bf Q}}

\def\algo{\textsc{PGS}}

\icmltitlerunning{Reliable Weakly Supervised Learning: Maximize Gain and Maintain Safeness}
\begin{document}
	\twocolumn[
	\icmltitle{Reliable Weakly Supervised Learning: Maximize Gain and Maintain Safeness}
	
	\icmlsetsymbol{equal}{*}
	
		\begin{icmlauthorlist}
		\icmlauthor{Lan-Zhe Guo}{to}
		\icmlauthor{Yu-Feng Li}{to}
		\icmlauthor{Ming Li}{to}
		\icmlauthor{Jin-Feng Yi}{ed}
		\icmlauthor{Bo-Wen Zhou}{ed}
		\icmlauthor{Zhi-Hua Zhou}{to}
	\end{icmlauthorlist}

\begin{center}
National Key Laboratory for Novel Software Technology, Nanjing University, Nanjing 210023, China\\
JD AI Research, China \\
\{guolz, liyf, lim, zhouzh\}@lamda.nju.edu.cn, \{yijinfeng, bowen.zhou\}@jd.com
\end{center}
	
	

	\icmlkeywords{Machine Learning, ICML}
	
	\vskip 0.3in
	]
	\begin{abstract}
		Weakly supervised data are widespread and have attracted much attention. However, since label quality is often difficult to guarantee, sometimes the use of weakly supervised data will lead to unsatisfactory performance, i.e., performance degradation or poor performance gains. Moreover, it is usually not feasible to manually increase the label quality, which results in weakly supervised learning being somewhat difficult to rely on. In view of this crucial issue, this paper proposes a simple and novel weakly supervised learning framework. We guide the optimization of label quality through a small amount of validation data, and to ensure the safeness of performance while maximizing performance gain. As validation set is a good approximation for describing generalization risk, it can effectively avoid the unsatisfactory performance caused by incorrect data distribution assumptions. We formalize this underlying consideration into a novel Bi-Level optimization and give an effective solution. Extensive experimental results verify that the new framework achieves impressive performance on weakly supervised learning with a small amount of validation data.   
	\end{abstract}
	
	\section{Introduction}
	
	Weakly supervised data~\cite{zhou2017brief} commonly appears in machine learning tasks due to the high cost in collecting of a large amount of strong supervision data. Examples include incomplete supervision data, i.e., only a small subset of training data is given with labels; inexact supervision data, i.e., only coarse-grained labels are given, and inaccurate supervision data, i.e., the given labels are not always ground-truth~\cite{zhou2017brief}. 
	Learning with weakly supervised data plays an important role in various applications, including text categorization~\cite{joachims1999transductive}, semantic segmentation~\cite{vezhnevets2012weakly}, bioinformatics~\cite{kasabov2004transductive}, natural language processing~\cite{huang2014learning}, medical care~\cite{wang2017chestx}, etc.
	
	A large number of efforts for weakly supervised learning (WSL) have been devoted~\cite{zhou2017brief,chapelle2006semi,settles2012active,frenay2014classification}, e.g., semi-supervised learning methods~\cite{chapelle2006semi}, label noise learning methods~\cite{frenay2014classification}, etc. Existing methods can roughly be categorized into two classes from the objective perspective. One is to maximize the performance gain based on some distribution assumption and the other is to maintain the performance safeness based on the worst-case analysis.
	
	The first objective of WSL focuses on maximizing the performance by making assumptions on the true label structure. For example, to maximize learning performance on unlabeled data, transductive support vector machines (TSVMs)~\cite{joachims1999transductive} adopt the low-density assumption~\cite{chapelle2005semi}, i.e., the decision boundary should lie in a low-density region. Label propagation algorithms~\cite{zhu2003semi} spread the soft labels from few labeled nodes to the whole graph based on the manifold assumptions~\cite{zhou2004learning}. To maximize learning performance with noisy labeled data, one can try to improve the label quality by assuming that the mislabeled items can be detected by measures like classification confidence~\cite{sun2007identifying}, model complexity~\cite{gamberger2000noise}, influence functions~\cite{koh2017understanding}.
	
	It is generally expected that in comparison with a supervised algorithm that uses only labeled data, WSL can help improve the performance by using more weakly supervised data. However, it is noteworthy that bad matching of problem structure with model assumption can lead to degradation in learning performance. The fact that WSL does not always help has been observed by multiple researchers,  \cite{cozman2003semi,chawla2005learning,chapelle2006semi,van2009knowledge,frenay2014classification,li2015towards} pointed out that there are cases in which the use of the weakly supervised data may degenerate the performance, making it even worse after WSL. 
	
	Therefore, the other objective of WSL that focus on maintaining the safeness has attracted significant attraction in recent years. \cite{li2011towards,li2015towards} builds safe semi-supervised SVMs through optimizing the worst-case performance gain given a set of candidate low-density separators. \cite{balsubramani2015optimally} proposes to learn a safe prediction given that the ground-truth label assignment is restricted to a specific candidate set. \cite{guo2018general} proposes a general safe WSL formulation given that the ground-truth can be constructed by a set of base learners and they optimize the worst-case performance gain. These methods are developed to alleviate the performance degeneration problem of WSL based on analyzing the worst-case of the ground-truth, and these efforts typically cause conservative performance improvement.

	\begin{figure}[t] 
		\centering
		\small
		\begin{tabular}{c@{ }@{ }c@{ }@{ }c}
			\begin{minipage}{1\linewidth}
				\includegraphics[width=1\textwidth]{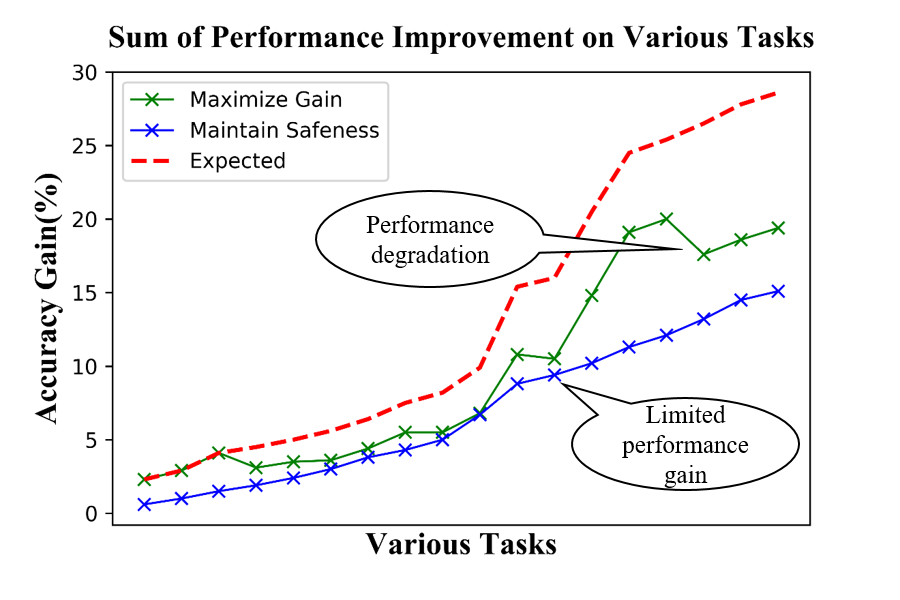}\\
			\end{minipage} 
		\end{tabular}
		\vspace{-.2cm}
		\caption{Illustration for the behaviors of three WSL methodologies: maximize gain, maintain safeness and user expected performance.}\label{fig:motivation}
		\vspace{-.4cm}
	\end{figure}
	
	However, in many practical tasks, we not only need a strong performance improvement but also need to ensure safeness, that is, we need to take into account both the best and the worst case of learning performance. Both are indispensable, because any dissatisfaction may not be what we hope to happen in the real applications, causing the results of weakly supervised learning being somewhat difficult to rely on. 
	
	In this work, we consider the question of how to optimize label quality for WSL such that both the best-case and the worst-case performance are considered. Figure~\ref{fig:motivation} illustrates the motivation for this work. Specifically, let $p_0$ be the performance derived from baseline supervised strategies without using weakly supervised data and $p_*$ be the performance derived from weakly supervised learning with correct data assumption. How to optimize the training set labels, such that the learned performance $p$ could be closely related to $p_*$, meanwhile it would not be worse than $p_0$. 
	
	In view of this issue, this paper proposes a simple and novel weakly supervised learning framework \algo \ (maximize Performance Gain while maintaining Safeness). Our framework uses a small amount of validation data to clearly guide label quality optimization. Because verification set is a good approximation for describing generalization risk, it can provide reliable and powerful guidance for performance improvement and performance safeness. At the same time, the verification set can effectively avoid the performance problems caused by incorrect data distribution assumptions. This new perspective of weakly supervised learning can be formalized into a novel Bi-Level optimization~\cite{bard2013practical} where one optimization problem is nested within another problem and can be effectively addressed for both convex and non-convex situations. A large number of experimental results also clearly confirm that the new framework achieves impressive performance. In summary, our contributions in this work are as follows:
	\begin{itemize}
		\item We propose a new weakly supervised learning framework that tries to maximize the performance gain and maintain safeness at the same time.
		\item We propose a reliable solution by formulating the problem into an optimization with effective algorithm, as well as justified by a large number of experiments. 
	\end{itemize}
	
	\begin{figure*}[t] 
		\centering
		\small
		\includegraphics[width=\textwidth]{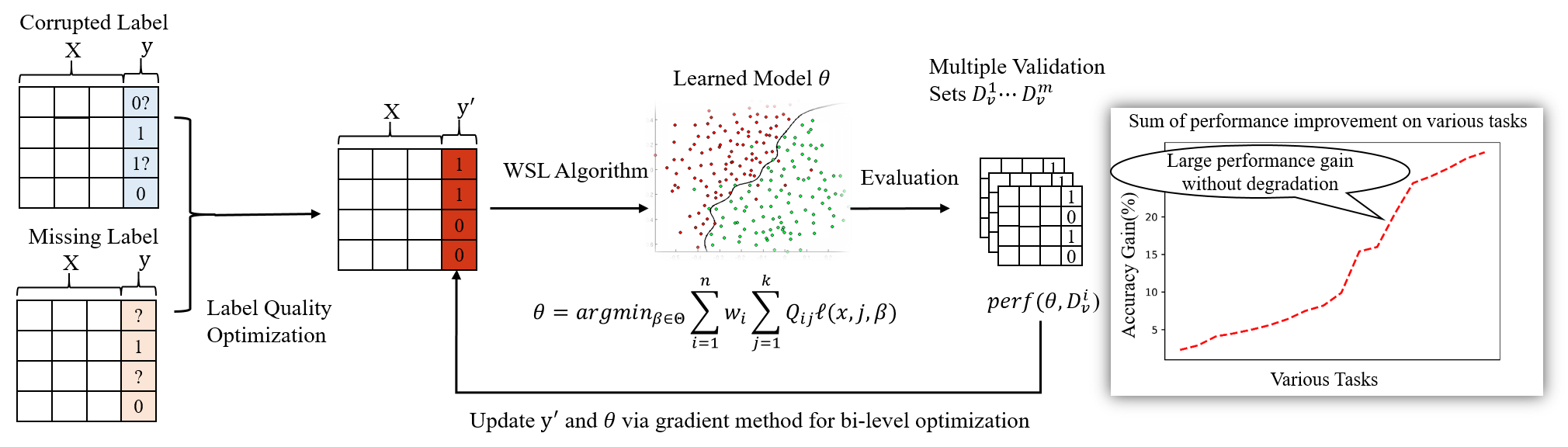}
		\caption{Illustration of the \algo\ framework. }\label{fig:framework}
	\end{figure*}
	In the following, we first review several related WSL frameworks and then present the \algo\ framework. Next, we show the experiment. Finally, we conclude this paper.
	
	\section{Two WSL Frameworks and Limitations} 
	\label{Proposed Method}
	
	\textbf{Notations and Setup}  Consider a prediction problem from some input space $\mathcal{X}$ (e.g., images) to an output space $\mathcal{Y}$ (e.g., labels). We are given a weakly supervised training data set $D_{tr} = \{\x_i, \y_i\}_{i=1}^{n}$ where $(\x_i, \y_i) \in \mathcal{X} \times \mathcal{Y}$. Assuming that the true data distribution is $D = \{\x, \y^*\}$. Let $\theta_0$ be the model derived from baseline supervised learning (SL) methods without WSL, i.e., $\theta_0 = \text{SL}(\x, \y)$ and $\theta$ be the model derived from WSL methods, i.e., $\theta = \text{WSL}(\x, \y)$.
	
	\textbf{Maximize Performance Gain} The first branch of WSL aims to maximize gain, i.e., 
	\begin{eqnarray}
	&\max\limits_\theta \max\limits_{\y'}&  perf(\theta, \x, \y^*) \\ \nonumber
	&\mbox{s.t.}&\theta = \text{WSL}(\x, \y') \nonumber
	\end{eqnarray}
	where $perf(\theta, \x, \y^*)$ is the performance of model $\theta$ on true data distribution $\{\x, \y^*\}$ which is unknown. It is assumed that the higher the performance, the better. 
	
	These methods~\cite{chapelle2005semi,sun2007identifying,gamberger2000noise,koh2017understanding} make assumptions on the true label structure and aim to maximize the best-case performance where best-case means the true label $\y^*$ follows assumptions of these WSL methods. However, mismatched assumption occurs in real applications and will cause performance degeneration problem~\cite{zhou2017brief}.
	
	\textbf{Maintain Performance Safeness} The second branch of WSL aims to maintain safeness, i.e., 
	\begin{eqnarray}
	&\max\limits_{\theta}\min\limits_{\y'}& perf(\theta, \x, \y^*) - perf(\theta_0, \x, \y^*)\\ \nonumber
	&\mbox{s.t.}& \theta_0= \text{SL}(\x, \y)\\ \nonumber
	&&\theta = \text{WSL}(\x, \y') \nonumber
	\end{eqnarray}
	These methods~\cite{li2011towards,li2015towards,balsubramani2015optimally,guo2018general} maximize the worst-case performance gain to maintain safeness, i.e., considering the true label $\y^*$ in the worst-case. The algorithm performs better than supervised learning methods while causing conservative performance improvement.
	
	\section{The \algo\ Framework}
	
	\subsection{General Formulation}
	
	The proposed \algo\ framework optimizes label quality for weakly supervised data by taking both the best-case and worst-case performance into account. The \algo\ maximizes learning performance of the model trained on the optimized weakly supervised data, meanwhile, \algo\ requires that the worst-case performance of the model won't be worse than the model trained on the raw label. 
	Thus, the \algo\ framework can be generally formulated as:
	\begin{eqnarray}
	&\max\limits_{\theta, \y'} & perf(\theta, \x, \y^*)\\ \nonumber
	&\mbox{s.t.}  & \min\limits_{\y^*}\{perf(\theta,  \x, \y^*) - perf(\theta_0,  \x, \y^*)\}\} \geq 0 \\ \nonumber
	&&\theta = \text{WSL}(\x, \y')\\ \nonumber
	&&\theta_0 = \text{SL}(\x, \y) \nonumber
	\end{eqnarray}
	Figure~\ref{fig:framework} illustrates the \algo\ framework.
	
	\subsection{Determine missing or corrupted labels robustly}
	
	In reality, the true label $\y^*$ is unavailable. To overcome this issue, \algo\ introduces a small clean unbiased validation set $D_{val} = \{\tilde{\x}, \tilde{\y}\}$ with size $n_v$ to approximate the true data distribution. The effectiveness of similar idea has been recently demonstrated in hyper-parameter tuning~\cite{maclaurin2015gradient,ravi2017optimization,ren2018meta,franceschi2017forward,franceschi2018bilevel}, meta-learning~\cite{ren2018learning}, few-shot learning~\cite{ravi2017optimization}, and training set debugging~\cite{cadamuro2016debugging,zhang2018training} for the reason that it is usually acceptable to manually check a small amount of training data as the validation set and the performance on a clean unbiased validation set well approximates the generalization performance.
	
	We adopt bootstrap strategy to create multiple validation sets $D_1^{v}, \cdots, D_{m}^{v}$ to determine the missing or corrupted training set labels robustly. On one hand,  the \algo\ optimizes validation performance of the model trained on the optimized label to maximize performance gain. On the other hand, \algo\ optimizes the performance gain against the original label on all validation sets simultaneously to maintain safeness. 
	\begin{eqnarray}
	\!\!\!\!&\max\limits_{\theta,\y'} & \sum_{i=1}^{m} perf(\theta, D_i^{v}) \; \text{\emph{(maximize gain)}} \\ \nonumber
	\!\!\!\!&\mbox{s.t.} & perf(\theta, D_i^{v}) \geq perf(\theta_0, D_i^{v})\;\; \text{\emph{(maintain safeness)}}\\ \nonumber
	\!\!\!\!&& i = 1, \dots, m \\ \nonumber
	\!\!\!\!&&\theta = \text{WSL}(\x, \y') \\ \nonumber
	\!\!\!\!&&\theta_0 = \text{SL}(\x, \y) 
	\end{eqnarray}
	
	In this paper, the worst-case is defined as the worst performance gain on one validation set. However, it is worth mentioning that in real applications, there are multiple flexible ways to define the worst-case performance to maintain safeness. For example, in fine-grained classification, if we need the model performance can not degrade on some specific classes, 
	the validation sets can be composed of examples in these classes. For multi-view/multi-modal data, we can require that the model performance will not decrease on every view/modal. We can also optimize multiple performance measures simultaneously and require that the performance will not decrease on every performance measure.
	
	\textbf{Maximize Gain}
	\algo\ adopts the validation performance as an objective to optimize the training label. By optimizing validation performance, the \algo\ can leverage useful information from a large training set, and still converge to an appropriate distribution favored by the validation set. A clean unbiased validation set is a good approximation of the true data distribution, thus, the proposed \algo\ framework can maximize performance gain and improve the generalization performance.
	
	\textbf{Maintain Safeness}
	\algo\ explicitly requires that in the worst-case, the model trained on the optimized data will not be worse than the model trained on the original data, i.e., do not have the performance degeneration problem. In this paper, we adopt the bootstrap strategy to create multiple validation sets and define the worst-case indicates the worst performance gain in these validation sets. As mentioned previously, in real applications, it is flexible to define the worst-case performance. Therefore, the \algo\ framework can maintain safeness in various scenarios.
	
	\subsection{Handling instances weights and label correction}
	
	Optimizing the label quality includes two operations. First is to determine the harmful instances. A weight $w\in [0, 1]$ is assigned to each training instance where higher weight corresponds to higher quality. Then, the harmful instance can be relabeled or simply keep the low weight to decrease their influence. To relabel the instance, we need to estimate a label transition quantity $\Q$ for this instance.
	
	With these two operations, the empirical risk minimization procedure for two fundamental machine learning tasks, classification and regression, can be written as:
	\begin{eqnarray}
	&&\theta =  \arg \min_{\beta \in \Theta}  \frac{1}{n}\sum_{i=1}^{n}w_i\sum_{j=1}^{k}Q_{ij}\ell(\x_i, j, \beta) \\
	&&\theta =  \arg \min_{\beta \in \Theta}  \frac{1}{n}\sum_{i=1}^{n}w_i\ell(\x_i, y_i + Q_i, \beta) 
	\end{eqnarray}
	where $k$ is the number of classes, $\w\in [0, 1]^n$ is the weight, $\Q\in \{Q_{ij} \geq 0, \sum_{j=1}^{k} Q_{ij}=1, \forall i\}$ for classification task and $\Q \in \mathbb{R}^n$ for regression task is the label transition quantity. The model trained on the raw data can be recovered when $\w = \textbf{1}$, $\Q_{i} = \e_{y_i}, \forall i$ for classification task and $\w = \textbf{1}$, $\Q= \textbf{0}$  for regression task.
	
	In order to reduce the complexity of the search space and utilize the information of original training data, \algo\ constrains the distance between the optimized label $\y'$ and the original label $\y$. For classification, the way to measure the distance between $\y'$ and $\y$ is the value of $1-Q_{i, y_i}$, for regression, the way to measure the distance between $\y'$ and $\y$ is the value of $ \|\Q\|_2 $.
	
	In practice, the performance of a model can be replaced with the surrogate loss function. Specifically, for a $k$ classification problem, our objective as follows:
	\begin{eqnarray}\label{eq:obj}
	&\min\limits_{(\w, \Q)\in \Lambda, \theta}& \frac{1}{m}\sum_{i=1}^{m} \sum_{j\in D_v^{i}}\ell(\tilde{\x}_{j}, \tilde{\y}_{j}, \theta)\\ \nonumber 
	&&+ \lambda \max\limits_{i\in{1,\cdots,m}}\left[\sum_{j\in D_v^{i}}\ell(\tilde{\x}_{j}, \tilde{\y}_{j}, \theta) - c_i \right] 
	\end{eqnarray}
	\begin{eqnarray}
	&\mbox{s.t.}&  \theta = \arg\min_{\beta \in \Theta} \frac{1}{n}\sum_{i=1}^{n}w_i\sum_{j=1}^{k}Q_{ij}\ell(\x_i, j, \beta) \nonumber
	\end{eqnarray}
	where $c_i$ is the validation loss of the model trained on the original data, $\Lambda = \{\w\in[0,1]^n, \|\w\|_1 \geq \epsilon_1; 0 \leq Q_{ij} \leq 1, \sum_{j=1}^{k}Q_{ij}=1, \forall i, \frac{1}{n}\sum_{i=1}^{n}(1 - Q_{i, y_i}) \leq \epsilon_2\}$.
	
	For regression task, the model training procedure can be replaced with $\theta =  \arg \min_{\beta \in \Theta}  \frac{1}{n}\sum_{i=1}^{n}w_i\ell(\x_i, y_i + Q_i, \beta)$ and $\Lambda = \{\w\in[0,1]^n, \|\w\|_1 \geq \epsilon_1; \|\Q\|_2 \leq \epsilon_2\}$.
	
	\subsection{Gradient Method for Bi-Level Optimization}
	
	In this section, we discuss the optimization strategies for \algo. The optimization of regression problem is similar to the classification one, thus we only discuss the detail procedure for the classification task. 
	
	Eq.(\ref{eq:obj}) is a bi-level optimization problem~\cite{bard2013practical}, where one optimization problem is nested within another problem.  The lower-level optimization is to find an empirical risk minimizer model given the training set whereas the upper-level optimization is to minimize the validation loss given the learned model. For the writing simplicity, we denote the upper-level objective as $\mathcal{L}_{\text{val}}(\tilde{\x}, \tilde{\y},\theta)$ and the lower-level objective as $\mathcal{L}_{\text{train}}(\theta, \w, \Q)$.
	
	It is difficult to optimize the upper-level objective function directly because in general there is no closed-form expression of the optimal $\theta$. The classical approaches for solving bi-level optimization problem can be categorized as single-level reduction methods, descent methods, trust-region methods, and evolutionary methods~\cite{sinha2018review}. In this paper, we adopt two of the most popular methods for convex and non-convex situations.
	
	\textbf{Optimization for Convexity} 
	If the empirical loss function is differentiable and strictly convex, the lower-level optimization problem can be replaced with its Karush-Kuhn-Tucker (KKT) conditions~\cite{boyd2004convex}:
	\begin{equation}\label{eq:kkt}
	\frac{\partial \mathcal{L}_{\text{train}}}{\partial\theta} = 
	\frac{1}{n}\sum_{i=1}^{n}w_i\sum_{j=1}^{k}Q_{ij}\nabla_{\theta}\ell(\x_i, \y_i, \theta) = 0
	\end{equation}
	The solution to $\frac{\partial \mathcal{L}_{\text{train}}}{\partial\theta} = 0$ defines an implicit function $\theta(\w,\Q)$. Then we can adopt the implicit function theorem to estimate how $\theta$ varies in $\w$ and $\Q$:
	\begin{eqnarray}
	&&\frac{\partial \theta}{\partial \w} = -\left(\frac{\partial^2\mathcal{L}_{\text{train}}}{\partial \theta \partial \theta^{\top}}\right)^{-1}\frac{\partial^2\mathcal{L}_{\text{train}}}{\partial\theta \partial\w^{\top}} \\ 
	&&\frac{\partial \theta}{\partial \Q} = -\left(\frac{\partial^2\mathcal{L}_{\text{train}}}{\partial \theta \partial \theta^{\top}}\right)^{-1}\frac{\partial^2\mathcal{L}_{\text{train}}}{\partial\theta \partial\Q^{\top}}
	\end{eqnarray}
	This tells us how $\theta$ changes with respect to an infinitesimal change to $\w$ and $\Q$. Now, we can apply the chain rule to get the gradient of the whole optimization problem w.r.t. $\w$ and $\Q$,
	\begin{eqnarray}\label{eq:gradient-convex}
	&&\frac{\partial \mathcal{L}_{\text{val}}}{\partial \w} =\frac{\partial \mathcal{L}_{\text{val}}}{\partial\theta}\left( -\left(\frac{\partial^2\mathcal{L}_{\text{train}}}{\partial \theta \partial \theta^{\top}}\right)^{-1}\frac{\partial^2\mathcal{L}_{\text{train}}}{\partial\theta \partial\w^{\top}}\right) \\ 
	&&\frac{\partial \mathcal{L}_{\text{val}}}{\partial \Q} = \frac{\partial \mathcal{L}_{\text{val}}}{\partial\theta}\left( -\left(\frac{\partial^2\mathcal{L}_{\text{train}}}{\partial \theta \partial \theta^{\top}}\right)^{-1}\frac{\partial^2\mathcal{L}_{\text{train}}}{\partial\theta \partial\Q^{\top}}\right)
	\end{eqnarray}
	
	In practice, the matrix inverses is not pleasing, the alternative method to compute $-(\frac{\partial^2\mathcal{L}_{\text{train}}}{\partial \theta \partial \theta^{\top}})^{-1}\frac{\partial^2\mathcal{L}_{\text{train}}}{\partial\theta \partial\w^{\top}}$ is compute it as the solution to $(\frac{\partial^2\mathcal{L}_{\text{train}}}{\partial \theta \partial \theta^{\top}})\x = -\frac{\partial^2\mathcal{L}_{\text{train}}}{\partial\theta \partial\w^{\top}}$. When the Hessian is positive definite, the linear system can be solved conveniently and only requires matrix-vector products, that is we do not have to materialize  the Hessian.
	
	\textbf{Optimization for Non-Convexity} However, in general, we can not use the implicit theorem to obtain the optimal $\theta$, for the reason that setting the derivative to zero only leads to a saddle point for non-convex functions. In general cases, we adopt gradient descent methods (or one of its variants like momentum, RMSProp, Adam, etc.) to solve the optimal $\theta$ approximately.  
	Specifically, the training procedure can be written as:
	\begin{equation}
	\theta_t = \theta_{t-1} - \eta (\nabla_\theta \mathcal{L}_{\text{train}}(\theta_{t-1},\w, \Q))
	\end{equation}
	where $\eta$ is the learning rate.
	We replace the lower-level problem with the dynamical procedure, and place them in the objective in the Lagrangian form:
	\begin{eqnarray}
	&&\mathcal{L}(\theta, \w, \Q, \alpha) =\mathcal{L}_{\text{val}}(\tilde{\x}, \tilde{\y},\theta_T) \\ \nonumber
	&&+ \sum_{t=1}^{T}\alpha_t(\theta_{t-1} - \eta (\nabla_{\theta}\mathcal{L}_{\text{train}}(\theta_{t-1},\w,\Q)-\theta_t)\nonumber
	\end{eqnarray}
	where $\alpha_t$ is the Lagrange multipliers associated with the $t$-step of the dynamical process. The partial derivatives of the Lagrangian are given by:
	\begin{eqnarray}
	&\frac{\partial\mathcal{L}}{\partial\alpha_t}=& \theta_{t-1} - \eta (\nabla_{\theta}\mathcal{L}_{\text{train}}(\theta_{t-1},\w,\Q))-\theta_t \\
	&\frac{\partial\mathcal{L}}{\partial \theta_t}=& \alpha_{t+1}A_{t+1} - \alpha_t, \;\;\; t\in{1, \cdots, T-1}\\
	&\frac{\partial\mathcal{L}}{\partial \theta_T}=&\nabla_{\theta}\mathcal{L}_{\text{val}}(\tilde{\x}, \tilde{\y},\theta_T) - \alpha_T\\
	&\frac{\partial\mathcal{L}}{\partial \w} =&\sum\nolimits_{t=1}^{T}\alpha_tB_t \\&\frac{\partial\mathcal{L}}{\partial \Q} =&\sum\nolimits_{t=1}^{T}\alpha_tC_t
	\end{eqnarray}
	where
	\begin{eqnarray}
	&A_t = &1 - \eta \left[\nabla_{\theta}\mathcal{L}_{\text{train}}(\theta_{t-1},\w,\Q)\right] \\ 
	&B_t = &-\eta\nabla_\w(\nabla_{\theta}\mathcal{L}_{\text{train}}(\theta_{t-1},\w,\Q))\\ 
	&C_t = &-\eta\nabla_\Q(\nabla_{\theta}\mathcal{L}_{\text{train}}(\theta_{t-1},\w,\Q))
	\end{eqnarray}
	By setting each derivative to zero we can obtain that,
	\begin{eqnarray}
	&\frac{\partial \mathcal{L}}{\partial \w} =& \nabla_{\theta} \mathcal{L}_{\text{val}}(\tilde{\x}, \tilde{\y}, \theta_T)\sum_{t=1}^{T}(\prod_{s=t+1}^{T}A_s)B_t\\
	&\frac{\partial \mathcal{L}}{\partial \Q} =& \nabla_{\theta} \mathcal{L}_{\text{val}}(\tilde{\x}, \tilde{\y}, \theta_T)\sum_{t=1}^{T}(\prod_{s=t+1}^{T}A_s)C_t 
	\end{eqnarray}
	Then, the whole problem can be solved with gradient methods. Algorithm~\ref{algo:convex} and~\ref{algo:non-convex} summarize the pseudo code of these two optimization strategies.
	
	\begin{algorithm}[!t]
		\caption{Optimization Algorithm for Convexity.}\label{algo:convex}
		\textbf{Input}: Training set $D_{tr} = \{\x_i, \y_i\}_{i=1}^{n}$, validation set  $D_{val} = \{\tilde{\x}_i, \tilde{\y}_i\}_{i=1}^{n^v}$, current values of $\w_0$ and $\Q_0$.
		
		\textbf{Output}: Learned weights $\w$ and label transition quantity $\Q$.
		\begin{algorithmic}[1]
			\FOR{$t = 1$ to $T$}
			\STATE Compute the gradient $\frac{\partial  \mathcal{L}_{\text{val}}}{\partial \theta}, \frac{\partial  \mathcal{L}_{\text{train}}}{\partial \theta}, \frac{\partial  \mathcal{L}_{\text{train}}}{\partial \w},\frac{\partial  \mathcal{L}_{\text{train}}}{\partial \Q}$.
			\STATE Compute the gradient $\frac{\partial  \mathcal{L}_{\text{val}}}{\partial \w},\frac{\partial  \mathcal{L}_{\text{val}}}{\partial \Q}$ using Eq.(\ref{eq:gradient-convex}).
			\STATE $\w_t = $ Optimization step($\w_{t-1}, \frac{\partial  \mathcal{L}_{\text{val}}}{\partial \w}$).
			\STATE $\Q_t = $ Optimization step($\Q_{t-1}, \frac{\partial  \mathcal{L}_{\text{val}}}{\partial \Q}$).
			\STATE $\w_t, \Q_t = \text{Projection step}(\w_t, \Q_t, \Lambda)$
			\ENDFOR
		\end{algorithmic}
	\end{algorithm}
	
	\begin{algorithm}[!t]
		\caption{Optimization Algorithm for Non-Convexity.}\label{algo:non-convex}
		\textbf{Input}: Training set $D_{tr} = \{\x_i, \y_i\}_{i=1}^{n}$, validation set  $D_{val} = \{\tilde{\x}_i, \tilde{\y}_i\}_{i=1}^{n^v}$, current values of $\w_0$ and $\Q_0$.
		
		\textbf{Output}: Learned weights $\w$ and label transition quantity $\Q$.
		\begin{algorithmic}[1]
			\FOR {$l = 1$ to $L$}
			\FOR{$t = 1$ to $T$}
			\STATE $\theta_t = \theta_{t-1} - \eta (\nabla_\theta \mathcal{L}_{\text{train}}(\theta_{t-1},\w_{l-1},\Q_{l-1}))$
			\ENDFOR
			\STATE $\alpha_T = \nabla_{\theta} \mathcal{L}_{\text{val}}(\tilde{\x}, \tilde{\y}, \theta_T)$  
			\STATE $g_{\w} = 0, g_{\Q} = 0$
			\FOR{$t = T-1$ to $1$}
			\STATE $g_\w = g_\w + \alpha_{t+1}B_{t+1}$
			\STATE $g_\Q = g_\Q+ \alpha_{t+1}C_{t+1}$ 
			\STATE $\alpha_t = \alpha_{t+1} A_{t+1}$
			\ENDFOR
			\STATE $\w_l = \text{Optimization step} (\w_{l-1}, g_\w)$
			\STATE $\Q_l = \text{Optimization step} (\Q_{l-1}, g_\Q)$
			\STATE $\w_l, \Q_l = \text{Projection step}(\w_l, \Q_l, \Lambda)$
			\ENDFOR
		\end{algorithmic}
	\end{algorithm}
	
	\textbf{Convergence}
	For convexity situations, the lower-level optimization problem is replaced with its KKT conditions and the overall bi-level optimization problem is reduced to a single-level optimization problem. Therefore, the optimization problem enjoys the same convergence properties as the gradient methods for single-level optimization.
	
	For non-convexity situations, we adopt an approximation procedure with respect to the original bi-level problem. It is necessary to analyze the convergence of this algorithm, where the proof is shown in the appendix.

	\begin{thm}\label{thm:convergence}
		(Convergence). Suppose the empirical loss function $\ell(\cdot, \cdot)$ is Lipschitz continuous. Let $\theta_{opt}$ be the optimal solution to the lower-level optimization problem, then as $T\to \infty$, we have $\arg\min_{(\w, \Q)} \mathcal{L_{\text{val}}}(\theta_T, \w, \Q) \to \arg\min_{(\w, \Q)}\mathcal{L_{\text{val}}}(\theta_{opt}, \w, \Q)$. 
	\end{thm}
	We stress that assumptions are very natural and satisfied by many loss functions of practical interests. For example, for classification, logistic loss is a Lipschitz smooth function. Similar cases can also be applied to square loss in regression.  
	
	\begin{table*}[!t] 
		\centering
		\caption{Average accuracy in terms of label noise learning on MNIST dataset under $50\%$ noisy ratio. The entries are bolded if the method achieves the best performance. The entries are boxed once the WSL method performs even worse than the baseline method.}
		\label{tbl:acc-mnist}
		\begin{tabular}{ccccccc}
			\hline
			Method & Baseline & Validation Only & REED &S-Model & SafeW & \algo\ \\
			\hline
			Unbiased validation set & 77.4 $\pm$ 0.45 & 76.2 $\pm$ 0.13 &78.6 $\pm$  0.21 &\fbox{76.2 $\pm$ 0.45} & 77.9 $\pm$ 0.34 & \textbf{83.4 $\pm$ 0.19}  \\
			Biased validation set &76.7 $\pm$ 0.47 & 69.1 $\pm$ 0.20 & 78.2 $\pm$ 0.30 &\fbox{75.7 $\pm$ 0.41} & 76.4 $\pm$ 0.42 & \textbf{80.3 $\pm$ 0.39} \\
			\hline
		\end{tabular}
	\end{table*}
	\begin{table*}[!t] 
		\centering
		\caption{Average accuracy in terms of semi-supervised learning on MNIST dataset with $40\%$ labeled data.}
		\label{tbl:ssl-mnist}
		\begin{tabular}{ccccccc}
			\hline
			Method & Baseline & $\Pi$-Model & Mean Teacher  &Pseudo-Labeling & SafeW & \algo\ \\
			\hline
			Unbiased validation set &  80.3 $\pm$ 0.43 & 83.6 $\pm$  0.36 &84.9 $\pm$ 0.40 & 82.5 $\pm$ 0.66 &82.2 $\pm$ 0.28 & \textbf{86.5 $\pm$ 0.33}  \\
			Biased validation set &79.8 $\pm$ 0.41 & 82.8 $\pm$ 0.49 &84.1 $\pm$ 0.41 & 82.0 $\pm$ 0.53 & 81.6 $\pm$ 0.36 &\textbf{84.1 $\pm$ 0.29} \\
			\hline
		\end{tabular}
	\end{table*}
	
	\section{Experiments}
	\label{Experiments}
	In order to validate the effectiveness of the proposed method, extensive experiments are conducted on a broad range of data sets that cover diverse domains including standard MNIST, CIFAR benchmarks for image classification, and six UCI datasets for regression tasks. Both unbiased and biased validation set are considered, and two WSL tasks, label noise learning and semi-supervised learning settings are conducted for comparison.  
	
	\subsection{Experimental Setup}
	For label noise learning, \algo\ is compared with the following methods.
	\textbf{REED}~\cite{DBLP:journals/corr/ReedLASER14}: a bootstrapping technique where the training target is a convex combination of the model prediction and the label. \textbf{S-MODEL}~\cite{goldberger2016training}: it adds a fully connected softmax layer after the regular classification output layer to model the noise transition matrix. \textbf{SafeW}~\cite{guo2018general}: it uses ensemble strategy to generate a prediction that can maintain safeness of weakly supervised learning. In addition, we compare with two simple baselines. \textbf{Baseline}: it combines the noisy data and validation data as the training set to train a model. \textbf{Validation only}: which only use the validation data as the training set to train a model.
	
	For semi-supervised learning, \algo\ is compared with the following methods.
	\textbf{$\Pi$-Model}~\cite{laine2016temporal}: it adds a loss term which encourages the distance between network's output for different passes of unlabeled data through the network to be small. \textbf{Mean Teacher}~\cite{tarvainen2017mean}: it is a stable version of $\Pi$-Model, which sets the target to predictions made using an exponential moving average of parameters from previous training steps. \textbf{Pseudo-Labeling}~\cite{lee2013pseudo}: it produces pseudo-labels for unlabeled data using the prediction function itself over the course of training.  In addition, we compare with \textbf{SafeW} and the baseline method \textbf{Baseline}.
	
	For the REED method, we use the best parameter reposted in~\cite{DBLP:journals/corr/ReedLASER14}. For the S-MODEL method, the transition weight is set to a smoothed identity matrix. The baseline method, REED, and S-Model are adopted as the base learners in SafeW. For \algo,
	a two-layer neural network is employed as the model. Gradient descent and ADAM are used to optimize the lower-level and the upper-level objective respectively. The iteration of lower-level and upper-level optimization are set to 500 and 20 for all experiments. 
	
	For the $\Pi$-Model, Mean Teacher, and Pseudo-Labeling method, we adopt the model structure and optimization method with \algo.
	The $\Pi$-Model, Mean Teacher, and Baseline method are adopted as the base learners in SafeW. For \algo, the weight $\w$ and label transition quantity $\Q$ for labeled data are known and we only optimize the value for unlabeled data.
	
	To make sure that our method does not have the privilege of training on more data, all compared methods combine the validation data with raw training data as a new training set.
	
	\subsection{MNIST Handwritten Digit Recognition Task}
	
	MNIST~\cite{lecun1998gradient}  is a standard dataset for handwritten digit classification. We select a total of $10,000$ images of size $28 \times 28$ and split into four subsets: the training set with 7,000 training examples, the validation set with 1,000 examples, the hyper-validation set with 1,000 examples to monitor training progress and tuning hyper-parameters, and the test data with $1,000$ examples. We also investigate the impact of the biased validation data by subsample a class imbalanced validation set. Specifically, the ratio between images belong to class 0-4 and 5-9 is shifted from 1:1 to 1:3 in the biased validation set, and 3 validation sets are generated with bootstrap strategy. 
	
	\textbf{Label Noise Learning Results} For label noise learning, we add uniform flip noise to the training set, means that all label classes can uniformly flip to any other label classes, which has been mostly studied in  literature~\cite{frenay2014classification}. The summary of classification accuracy with noisy ratio $50\%$ is reported in Table~\ref{tbl:acc-mnist}. All the results are averaged from 5 runs with different random splits of datasets. 
	
	From the comparison results in Table~\ref{tbl:acc-mnist}, our method achieves the best performance on both biased and unbiased validation sets among all the methods. Methods that maximize performance only, such as S-Model, could suffer performance degradation. The overall performance improvement of safeness-only methods (such as SafeW) is limited. Our approach is not inferior to the baseline approach and achieves maximum performance improvement. It is not effective to use validation sets only, because the amount of data in validation sets is small and it is difficult to train a good model. Above results demonstrate that \algo\ is not equivalent to the baseline methods which simply train a model on the validation and training data. In contrast, \algo\  utilizes the validation data to improve the label quality of training set and thus improve the final performance. 
	
	\textbf{Semi-Supervised Learning Results} For semi-supervised learning,  $40\%$ of the training data is labeled while the rest are unlabeled. Similar to label noise learning, the summary of classification accuracy is reported in Table~\ref{tbl:ssl-mnist} and results are averaged from 5 runs with different random splits. 
	
	Results in Table~\ref{tbl:ssl-mnist} show that \algo\ also achieves good performance on semi-supervised learning. The key difference between semi-supervised learning and label noise learning is that we know which instance has a high-quality label in semi-supervised learning, thus, \algo\ achieves better performance with less trusted labeled training instances.
	
	\subsection{CIFAR-10 Image Classification Task}
	
	CIFAR-10~\cite{lecun1998gradient} is benchmark for image classification task. The dataset consists of natural images with a size of $32 \times 32$ pixels and has 10 categories. We also subsample a set of $10,000$ images and split into four subsets: the training set with 7,000 training examples, the validation set with 1,000 examples, the hyper-validation set with 1,000 examples, and the test data with $1,000$ examples. To construct a biased validation set, the ratio between the top five classes and the bottom five classes is shifted to 1:3, and we create 3 validation sets using bootstrap strategy.
	
	\textbf{Label Noise Learning Results} The summary of classification accuracy under $50\%$ noisy ratio for label noise learning with uniform flip noise is reported in Table~\ref{tbl:acc-cifar}. All the results are averaged from 5 runs with different random splits of datasets.  Results on Table~\ref{tbl:acc-cifar} show that \algo\ also obtains maximum performance improvement on label noise learning and do not have performance degradation problem.
	
	\begin{table*}[!t] 
		\centering
		\caption{Average accuracy in terms of label noise learning on CIFAR-10 dataset under 50\% noisy ratio.}
		\label{tbl:acc-cifar}
		\begin{tabular}{ccccccc}
			\hline
			Method & Baseline & Validation Only & REED &S-Model & SafeW & \algo\ \\
			\hline
			Unbiased validation set & 58.3 $\pm$ 0.43 & 49.5 $\pm$ 0.24 & 64.5 $\pm$ 0.39 & 63.3 $\pm$ 0.87 & 62.0 $\pm$ 0.42 & \textbf{66.3 $\pm$ 0.13}  \\
			Biased validation set & 56.8 $\pm$ 0.47 & 42.3 $\pm$ 0.20 & 64.3 $\pm$ 0.41 & 62.9 $\pm$ 0.67 & 61.5 $\pm$ 0.44 &\textbf{64.5 $\pm$ 0.33}\\
			\hline
		\end{tabular}
	\end{table*}
	
	\begin{table*}[!t] 
		\centering
		\caption{Average accuracy in terms of semi-supervised learning on CIFAR-10 dataset with $40\%$ labeled data.}
		\label{tbl:ssl-cifar}
		\begin{tabular}{ccccccc}
			\hline
			Method & Baseline & $\Pi$-Model & Mean Teacher  &Pseudo-Labeling & SafeW & \algo\ \\
			\hline
			Unbiased validation set & 60.7 $\pm$ 0.38 &64.4 $\pm$ 0.53 & 65.7 $\pm$ 0.28 & 62.1 $\pm$ 0.50 & 63.5 $\pm$ 0.39 & \textbf{68.8 $\pm$ 0.33}  \\
			Biased validation set & 59.0 $\pm$ 0.38 & 63.6 $\pm$ 0.49 &64.9 $\pm$ 0.25 & 61.8 $\pm$ 0.43 & 62.5 $\pm$ 0.41 &\textbf{65.7 $\pm$ 0.40}\\
			\hline
		\end{tabular}
	\end{table*}

	\textbf{Semi-Supervised Learning Results} The results of classification accuracy with $40\%$ labeled data for semi-supervised learning is reported in Table~\ref{tbl:ssl-cifar}. All the results are averaged from 5 runs with different random splits of datasets.  Results on Table~\ref{tbl:ssl-cifar} show that \algo\ also derive highly competitive performance with all compared methods on semi-supervised learning.
	
	\begin{table}[!t]\vspace{-4mm}
		\centering
		\caption{Mean Squared Error of regression task on six UCI datasets}
		\label{tbl:reg}
		\rowcolors{2}{white}{gray!10}
		\begin{tabular}{cccc}
			\hline\hline
			\multicolumn{4}{c}{Label Noise Learning} \\
			\hline
			Dataset & Baseline & SVR & \algo  \\
			\hline
			abalone & .017 $\pm$ .004 &\fbox{.143 $\pm$ .033}&\textbf{.010 $\pm$ .002}\\
			bodyfat & .150 $\pm$ .016 &.102 $\pm$ .013&\textbf{.087 $\pm$ .010} \\
			cpusmall & .005 $\pm$ .001&.003 $\pm$ .001&\textbf{.002 $\pm$ .001}\\
			mg & .033 $\pm$ .004 &.030 $\pm$ .005&\textbf{.026 $\pm$ .004} \\
			mpg & .081 $\pm$ .011 &.074 $\pm$ .009&\textbf{.021 $\pm$ .005}\\
			space\_ga & .100 $\pm$ .014 &.093 $\pm$ .010&\textbf{.021 $\pm$ .006}\\
			\hline\hline
			\multicolumn{4}{c}{Semi-Supervised Learning} \\
			\hline
			Dataset & Baseline & Self-LS & \algo \\
			\hline
			abalone & .015 $\pm$ .003 & .012 $\pm$ .020&\textbf{.011 $\pm$ .002}\\
			bodyfat & .089 $\pm$ .023 &.062 $\pm$ .019&\textbf{.053 $\pm$ .011} \\
			cpusmall & .006 $\pm$ .001&.003 $\pm$ .001&\textbf{.002 $\pm$ .001}\\
			mg & .039 $\pm$ .004 &.031 $\pm$ .006&\textbf{.025 $\pm$ .008} \\
			mpg & .084 $\pm$ .013 &.040 $\pm$ .011&\textbf{.029 $\pm$ .003}\\
			space\_ga & .060 $\pm$ .009 &.031 $\pm$ .010&\textbf{.019 $\pm$ .004}\\
			\hline\hline
		\end{tabular}\vspace{-4mm}
	\end{table} 
	
	\subsection{Datasets for Regression Task}
	
	We also do experiments on six regression datasets to verifies the effectiveness of our proposal on regression tasks. The datasets cover diverse domains including physical measurements (\emph{abalone}), health (\emph{bodyfat}), economics (\emph{cadata}), activity recognition (\emph{mpg}), etc. For each dataset, we split it into four parts: training set, validation set, hyper-validation set, test set, according to the ratio $7 : 1 : 1 : 1$. We normalize all features and labels into $[0, 1]$. For label noise learning, we add a gauss noise to the training set, and for semi-supervised learning, we select 40\% of training set as the labeled data. 
	
	Since there are rarely works focusing on weakly supervised regression tasks, for label noise learning, we only compared with the baseline method which uses the training set and validation set simultaneously and a noisy robust regression method, Support Vector Regression (SVR). For semi-supervised learning, we compared with the baseline method and the Self-LS method which is an extension of supervised least square method based on self-training. For \algo, we use linear regression as the model, i.e., $\theta = \arg\min \| \y - \theta^{\top}\x \|^2 _2$.
	
	\textbf{Label Noise Learning Results} We report the Mean Squared Error for these UCI datasets in Table~\ref{tbl:reg}. Each result is averaged of 10 runs under 50\% noisy ratio. From Table~\ref{tbl:reg}, we can see that our proposal achieves the largest improvement, even when SVR suffers the performance degeneration problem. These demonstrate the effectiveness of \algo\ for label noise regression task.
	
	\textbf{Semi-Supervised Learning Results} The results of semi-supervised learning are reported in Table~\ref{tbl:reg}. Each result is averaged of 10 runs with 40\% labeled data. From Table~\ref{tbl:reg}, we can realize that our proposal consistently achieves highly competitive performance on semi-supervised regression.
	
	\subsection{Label Quality Improvement}
	
	To further test the effectiveness of label quality improvement,
	we compare \algo\ with two methods: Influence function~\cite{koh2017understanding}, which corrects label via perturbing a training point and counting the changes to prediction; Nearest Neighbor, which recommends the label from closest validation data when asked for a suggested label correction. These two methods are two commonly used as baselines for mislabel correction~\cite{zhang2018training}. 
	
	From Table~\ref{tbl:effect_of_detect}, we can see that \algo\ dominates the compared methods, which further demonstrates the effectiveness of our proposal in label quality improvement.
	
	\begin{table}[!h]\vspace{-4mm}
		\centering
		\caption{F1-Score of mislabeled data correction on MNIST dataset with noisy ratio varies from 10\% to 60\%.}
		\label{tbl:effect_of_detect}
		\begin{tabular}{cccc}
			\hline\hline
			Noisy  Ratio&  \tabincell{c}{Nearest\\ Neighbor} & \tabincell{c}{Influence \\ Function} & \algo \\
			\hline
			10\% & 54.17 &  65.71 & \textbf{68.21}\\
			20\% & 56.56 & 65.79 &\textbf{68.06} \\
			30\% & 58.32 & 65.21 &\textbf{67.64} \\
			40\% & 60.16 & 66.81 &\textbf{68.33} \\
			50\% & 61.15 & 65.05 &\textbf{67.97} \\
			60\% & 60.15 & 64.81 &\textbf{66.91} \\
			
			\hline\hline
		\end{tabular}\vspace{-4mm}
	\end{table}
	
\subsection{Validation Set Size}
It is meaningful to study the impact of validation set size. We investigate how the classification accuracy varies with the size of validation set increased from 100 to 1000 on MNIST dataset. We design experiments of \algo\ on both label noise learning with 50\% noisy ratio and semi-supervised learning with 40\% labeled data. The accuracy improvement against baseline method is plotted in Figure~\ref{fig:validation-size}. 
	
From Figure~\ref{fig:validation-size}, we can see that, as long as the size of validation set reaches a certain small scale (such as 100 samples), the proposal already achieves quite good performance improvement. Moreover, our proposal has great potential. Although we have achieved good results, this experiment shows that by increasing the size of validation set, we can continue to improve the performance significantly.

	\begin{figure}[!h]
		\centering
		\includegraphics[width=0.5\textwidth]{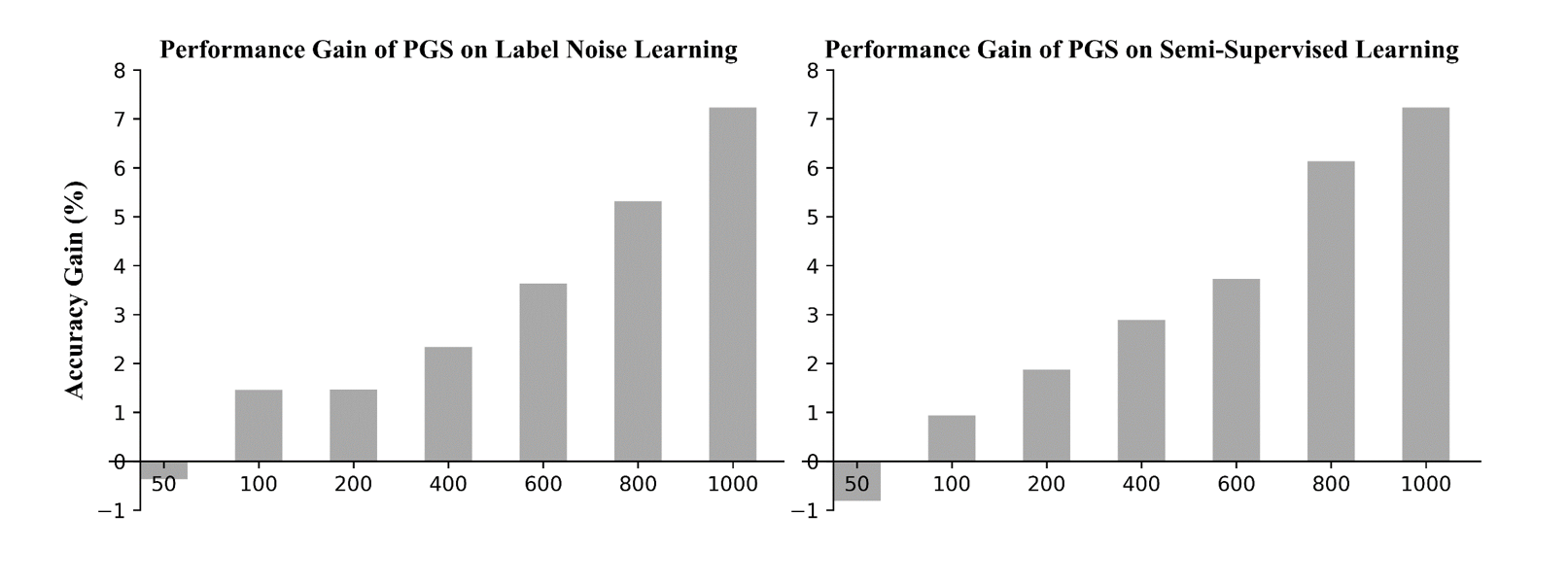}  
		\caption{Performance gain of \algo\ with varied validation set sizes on label noise learning and semi-supervised learning on MNIST.}
		\label{fig:validation-size}
	\end{figure}
	
\section{Conclusion}
	\label{Discussionss}
	
	This paper presents a new WSL framework. Compared with previous works, it considers both 1) performance maximization and 2) performance safeness. These two requirements are usually indispensable in many practical tasks. We use small-scale validation set to guide label quality optimization to get an effective solution, and the effectiveness is clearly verified by extensive experimental results. We believe that the new framework opens a door for reliable weakly supervised learning. Many follow-up works, such as integration of the adversarial mechanism, can be further studied.
	
	\bibliographystyle{icml2019}
	\bibliography{reference.bib}

\begin{thebibliography}{39}
\providecommand{\natexlab}[1]{#1}
\providecommand{\url}[1]{\texttt{#1}}
\expandafter\ifx\csname urlstyle\endcsname\relax
  \providecommand{\doi}[1]{doi: #1}\else
  \providecommand{\doi}{doi: \begingroup \urlstyle{rm}\Url}\fi

\bibitem[Balsubramani \& Freund(2015)Balsubramani and
  Freund]{balsubramani2015optimally}
Balsubramani, A. and Freund, Y.
\newblock Optimally combining classifiers using unlabeled data.
\newblock In \emph{Proceedings of the 28th Conference on Learning Theory}, pp.\
   211--225, 2015.

\bibitem[Bard(2013)]{bard2013practical}
Bard, J.~F.
\newblock \emph{Practical bilevel optimization: algorithms and applications}.
\newblock Springer Science \& Business Media, 2013.

\bibitem[Boyd \& Vandenberghe(2004)Boyd and Vandenberghe]{boyd2004convex}
Boyd, S. and Vandenberghe, L.
\newblock \emph{Convex optimization}.
\newblock Cambridge University Press, 2004.

\bibitem[Cadamuro \& Zhu(2016)Cadamuro and Zhu]{cadamuro2016debugging}
Cadamuro, G. and Zhu, X.
\newblock Debugging machine learning models.
\newblock In \emph{ICML Workshop on Reliable Machine Learning in the Wild},
  2016.

\bibitem[Chapelle \& Zien(2005)Chapelle and Zien]{chapelle2005semi}
Chapelle, O. and Zien, A.
\newblock Semi-supervised classification by low density separation.
\newblock In \emph{Proceedings of the 10th International Workshop on Artificial
  Intelligence and Statistics}, pp.\  57--64, 2005.

\bibitem[Chapelle et~al.(2006)Chapelle, Scholkopf, and Zien]{chapelle2006semi}
Chapelle, O., Scholkopf, B., and Zien, A.
\newblock \emph{Semi-supervised learning}.
\newblock MIT Press, 2006.

\bibitem[Chawla et~al.(2005)]{chawla2005learning}
Chawla, N.~V. et~al.
\newblock Learning from labeled and unlabeled data: An empirical study across
  techniques and domains.
\newblock \emph{Journal of Artificial Intelligence Research}, pp.\  331--366,
  2005.

\bibitem[Cozman et~al.(2003)Cozman, Cohen, and Cirelo]{cozman2003semi}
Cozman, F.~G., Cohen, I., and Cirelo, M.~C.
\newblock Semi-supervised learning of mixture models.
\newblock In \emph{Proceedings of the 20th International Conference on Machine
  Learning}, pp.\  99--106, 2003.

\bibitem[Franceschi et~al.(2017)Franceschi, Donini, Frasconi, and
  Pontil]{franceschi2017forward}
Franceschi, L., Donini, M., Frasconi, P., and Pontil, M.
\newblock Forward and reverse gradient-based hyperparameter optimization.
\newblock In \emph{Proceedings of the 34th International Conference on Machine
  Learning}, pp.\  1165--1173, 2017.

\bibitem[Franceschi et~al.(2018)Franceschi, Frasconi, Salzo, and
  Pontil]{franceschi2018bilevel}
Franceschi, L., Frasconi, P., Salzo, S., and Pontil, M.
\newblock Bilevel programming for hyperparameter optimization and
  meta-learning.
\newblock In \emph{Proceedings of the 35th International Conference on Machine
  Learning}, pp.\  1563--1572, 2018.

\bibitem[Fr{\'e}nay \& Verleysen(2014)Fr{\'e}nay and
  Verleysen]{frenay2014classification}
Fr{\'e}nay, B. and Verleysen, M.
\newblock Classification in the presence of label noise: a survey.
\newblock \emph{IEEE Transactions on Neural Networks and Learning Systems},
  pp.\  845--869, 2014.

\bibitem[Gamberger et~al.(2000)Gamberger, Lavrac, and
  Dzeroski]{gamberger2000noise}
Gamberger, D., Lavrac, N., and Dzeroski, S.
\newblock Noise detection and elimination in data preprocessing: experiments in
  medical domains.
\newblock \emph{Applied Artificial Intelligence}, pp.\  205--223, 2000.

\bibitem[Goldberger \& Ben-Reuven(2017)Goldberger and
  Ben-Reuven]{goldberger2016training}
Goldberger, J. and Ben-Reuven, E.
\newblock Training deep neural-networks using a noise adaptation layer.
\newblock In \emph{Proceedings of the 5th International Conference on Learning
  Representations}, 2017.

\bibitem[Guo \& Li(2018)Guo and Li]{guo2018general}
Guo, L.-Z. and Li, Y.-F.
\newblock A general formulation for safely exploiting weakly supervised data.
\newblock In \emph{Proceedings of the 32nd AAAI Conference on Artificial
  Intelligence}, pp.\  3126--3133, 2018.

\bibitem[Huang et~al.(2014)Huang, Ahuja, Downey, Yang, Guo, and
  Yates]{huang2014learning}
Huang, F., Ahuja, A., Downey, D., Yang, Y., Guo, Y., and Yates, A.
\newblock Learning representations for weakly supervised natural language
  processing tasks.
\newblock \emph{Computational Linguistics}, pp.\  85--120, 2014.

\bibitem[Joachims(1999)]{joachims1999transductive}
Joachims, T.
\newblock Transductive inference for text classification using support vector
  machines.
\newblock In \emph{Proceedings of the 16th International Conference on Machine
  Learning}, pp.\  200--209, 1999.

\bibitem[Kasabov \& Pang(2004)Kasabov and Pang]{kasabov2004transductive}
Kasabov, N. and Pang, S.
\newblock Transductive support vector machines and applications in
  bioinformatics for promoter recognition.
\newblock \emph{Neural Information Processing-Letters and Reviews}, pp.\  1--6,
  2004.

\bibitem[Koh \& Liang(2017)Koh and Liang]{koh2017understanding}
Koh, P.~W. and Liang, P.
\newblock Understanding black-box predictions via influence functions.
\newblock In \emph{Proceedings of the 34th International Conference on Machine
  Learning}, pp.\  1885--1894, 2017.

\bibitem[Laine \& Aila(2017)Laine and Aila]{laine2016temporal}
Laine, S. and Aila, T.
\newblock Temporal ensembling for semi-supervised learning.
\newblock In \emph{Proceedings of the 5th International Conference on Learning
  Representations}, 2017.

\bibitem[LeCun et~al.(1998)LeCun, Bottou, Bengio, and
  Haffner]{lecun1998gradient}
LeCun, Y., Bottou, L., Bengio, Y., and Haffner, P.
\newblock Gradient-based learning applied to document recognition.
\newblock In \emph{Proceedings of the IEEE}, pp.\  2278--2324, 1998.

\bibitem[Lee(2013)]{lee2013pseudo}
Lee, D.-H.
\newblock Pseudo-label: The simple and efficient semi-supervised learning
  method for deep neural networks.
\newblock In \emph{ICML Workshop on Challenges in Representation Learning},
  pp.\  2--8, 2013.

\bibitem[Li \& Zhou(2011)Li and Zhou]{li2011towards}
Li, Y.-F. and Zhou, Z.-H.
\newblock Towards making unlabeled data never hurt.
\newblock In \emph{Proceedings of the 28th International Conference on Machine
  Learning}, pp.\  1081--1088, 2011.

\bibitem[Li \& Zhou(2015)Li and Zhou]{li2015towards}
Li, Y.-F. and Zhou, Z.-H.
\newblock Towards making unlabeled data never hurt.
\newblock \emph{IEEE Transactions on Pattern Analysis and Machine
  Intelligence}, pp.\  175--188, 2015.

\bibitem[Maclaurin et~al.(2015)Maclaurin, Duvenaud, and
  Adams]{maclaurin2015gradient}
Maclaurin, D., Duvenaud, D., and Adams, R.
\newblock Gradient-based hyperparameter optimization through reversible
  learning.
\newblock In \emph{Proceedings of the 32nd International Conference on Machine
  Learning}, pp.\  2113--2122, 2015.

\bibitem[Ravi \& Larochelle(2017)Ravi and Larochelle]{ravi2017optimization}
Ravi, S. and Larochelle, H.
\newblock Optimization as a model for few-shot learning.
\newblock In \emph{Proceedings of the 5th International Conference on Learning
  Representations}, 2017.

\bibitem[Reed et~al.(2014)Reed, Lee, Anguelov, Szegedy, Erhan, and
  Rabinovich]{DBLP:journals/corr/ReedLASER14}
Reed, S.~E., Lee, H., Anguelov, D., Szegedy, C., Erhan, D., and Rabinovich, A.
\newblock Training deep neural networks on noisy labels with bootstrapping.
\newblock \emph{CoRR}, abs/1412.6596, 2014.

\bibitem[Ren et~al.(2018{\natexlab{a}})Ren, Triantafillou, Ravi, Snell,
  Swersky, Tenenbaum, Larochelle, and Zemel]{ren2018meta}
Ren, M., Triantafillou, E., Ravi, S., Snell, J., Swersky, K., Tenenbaum, J.~B.,
  Larochelle, H., and Zemel, R.~S.
\newblock Meta-learning for semi-supervised few-shot classification.
\newblock In \emph{Proceedings of the 6th International Conference on Learning
  Representations}, 2018{\natexlab{a}}.

\bibitem[Ren et~al.(2018{\natexlab{b}})Ren, Zeng, Yang, and
  Urtasun]{ren2018learning}
Ren, M., Zeng, W., Yang, B., and Urtasun, R.
\newblock Learning to reweight examples for robust deep learning.
\newblock In \emph{Proceedings of the 35th International Conference on Machine
  Learning}, pp.\  4331--4340, 2018{\natexlab{b}}.

\bibitem[Settles(2012)]{settles2012active}
Settles, B.
\newblock Active learning.
\newblock \emph{Synthesis Lectures on Artificial Intelligence and Machine
  Learning}, pp.\  1--114, 2012.

\bibitem[Sinha et~al.(2018)Sinha, Malo, and Deb]{sinha2018review}
Sinha, A., Malo, P., and Deb, K.
\newblock A review on bilevel optimization: from classical to evolutionary
  approaches and applications.
\newblock \emph{IEEE Transactions on Evolutionary Computation}, pp.\  276--295,
  2018.

\bibitem[Sun et~al.(2007)Sun, Zhao, Wang, and Chen]{sun2007identifying}
Sun, J.-w., Zhao, F.-y., Wang, C.-j., and Chen, S.-f.
\newblock Identifying and correcting mislabeled training instances.
\newblock In \emph{Future Generation Communication and Networking}, pp.\
  244--250, 2007.

\bibitem[Tarvainen \& Valpola(2017)Tarvainen and Valpola]{tarvainen2017mean}
Tarvainen, A. and Valpola, H.
\newblock Mean teachers are better role models: Weight-averaged consistency
  targets improve semi-supervised deep learning results.
\newblock In \emph{Advances in Neural Information Processing Systems}, pp.\
  1195--1204, 2017.

\bibitem[Van~Hulse \& Khoshgoftaar(2009)Van~Hulse and
  Khoshgoftaar]{van2009knowledge}
Van~Hulse, J. and Khoshgoftaar, T.
\newblock Knowledge discovery from imbalanced and noisy data.
\newblock \emph{IEEE Transactions on Knowledge and Data Engineering}, pp.\
  1513--1542, 2009.

\bibitem[Vezhnevets et~al.(2012)Vezhnevets, Ferrari, and
  Buhmann]{vezhnevets2012weakly}
Vezhnevets, A., Ferrari, V., and Buhmann, J.~M.
\newblock Weakly supervised structured output learning for semantic
  segmentation.
\newblock In \emph{Proceedings of the 25th IEEE Conference on Computer Vision
  and Pattern Recognition}, pp.\  845--852, 2012.

\bibitem[Wang et~al.(2017)Wang, Peng, Lu, Lu, Bagheri, and
  Summers]{wang2017chestx}
Wang, X., Peng, Y., Lu, L., Lu, Z., Bagheri, M., and Summers, R.~M.
\newblock Chestx-ray8: Hospital-scale chest x-ray database and benchmarks on
  weakly-supervised classification and localization of common thorax diseases.
\newblock In \emph{Proceedings of the 30th IEEE Conference on Computer Vision
  and Pattern Recognition}, pp.\  3462--3471, 2017.

\bibitem[Zhang et~al.(2018)Zhang, Zhu, and Wright]{zhang2018training}
Zhang, X., Zhu, X., and Wright, S.~J.
\newblock Training set debugging using trusted items.
\newblock In \emph{Proceedings of the 32nd AAAI Conference on Artificial
  Intelligence}, pp.\  4482--4489, 2018.

\bibitem[Zhou et~al.(2004)Zhou, Bousquet, Lal, Weston, and
  Sch{\"o}lkopf]{zhou2004learning}
Zhou, D., Bousquet, O., Lal, T.~N., Weston, J., and Sch{\"o}lkopf, B.
\newblock Learning with local and global consistency.
\newblock In \emph{Advances in Neural Information Processing Systems}, pp.\
  321--328, 2004.

\bibitem[Zhou(2017)]{zhou2017brief}
Zhou, Z.-H.
\newblock A brief introduction to weakly supervised learning.
\newblock \emph{National Science Review}, pp.\  44--53, 2017.

\bibitem[Zhu et~al.(2003)Zhu, Ghahramani, and Lafferty]{zhu2003semi}
Zhu, X., Ghahramani, Z., and Lafferty, J.~D.
\newblock Semi-supervised learning using gaussian fields and harmonic
  functions.
\newblock In \emph{Proceedings of the 20th International Conference on Machine
  learning}, pp.\  912--919, 2003.

\end{thebibliography}
	
	\newpage
	\twocolumn[
	\icmltitle{Supplementary}
	]
	
		\section*{A. Proof of Theorem 1}
	\begin{thm}\label{thm:convergence}
		(Convergence). Suppose the empirical loss function $\ell(\cdot, \cdot)$ is Lipschitz continuous. Let $\theta_{opt}$ be the optimal solution to the lower-level optimization problem, then as $T\to \infty$, we have $\arg\min_{(\w, \Q)} \mathcal{L_{\text{val}}}(\theta_T, \w, \Q) \to \arg\min_{(\w, \Q)}\mathcal{L_{\text{val}}}(\theta_{opt}, \w, \Q)$. 
	\end{thm}
	
	\begin{proof}
		Since the $\ell(\cdot, \cdot)$ is Lipschitz continuous, there exists $\rho > 0$ such that for every $t \in \mathcal{N}$ and every $(\w, \Q)$,
		\begin{equation*}
		\| \mathcal{L_{\text{val}}}(\theta_T, \w, \Q) - \mathcal{L_{\text{val}}}(\theta_{opt}, \w, \Q)\| \leq \rho \| \theta_T - \theta_{opt} \|
		\end{equation*}
		Let
		\begin{equation*}
		\w_T, \Q_T = \arg\min_{\w,\Q}\mathcal{L_{\text{val}}}(\theta_T, \w, \Q)
		\end{equation*}
		
		With $T \to\infty$, it is obvious that $\mathcal{L_{\text{val}}}(\theta_T, \w, \Q)\to \mathcal{L_{\text{val}}}(\theta_{opt}, \w, \Q)$, thus we have that $\| \mathcal{L_{\text{val}}}(\theta_T, \w_T, \Q_T) - \mathcal{L_{\text{val}}}(\theta_{opt}, \w_T, \Q_T) \| \to 0$.
		\begin{eqnarray*}
			&&\lvert L(\theta_T, (w_T, \delta_T)) - L(\theta_{optimal}, (w_T, \delta_T)) \rvert \\\nonumber
			&&\leq \sup_{w, \delta}\lvert L(\theta_T, (w, \delta))  - L(\theta_{optimal}, (w, \delta)) \rvert \\ \nonumber
		\end{eqnarray*}
		
		Using the continuity of $\ell(\cdot, \cdot)$, we have $\forall (\w, \Q)$,
		\begin{eqnarray*}
			\lim_{T\to\infty}\mathcal{L_{\text{val}}}(\theta_{opt}, \w_T, \Q_T)  =\lim_{T\to\infty}\mathcal{L_{\text{val}}}(\theta_T, \w_T, \Q_T)\\
			\leq \lim_{T\to\infty}\mathcal{L_{\text{val}}}(\theta_T, \w, \Q) =\mathcal{L_{\text{val}}}(\theta_{opt}, \w, \Q) 
		\end{eqnarray*}
		Therefore, we have
		\begin{equation*}
		(\w_T, \Q_T) = \arg\min_{\w, \Q} \mathcal{L_{\text{val}}}(\theta_{opt}, \w, \Q).
		\end{equation*}
		which completes the proof.
	\end{proof}
	
	\section*{B. Impact of Iteration Number}
	
	\begin{figure}[!h]
		\centering
		\includegraphics[width=0.5\textwidth]{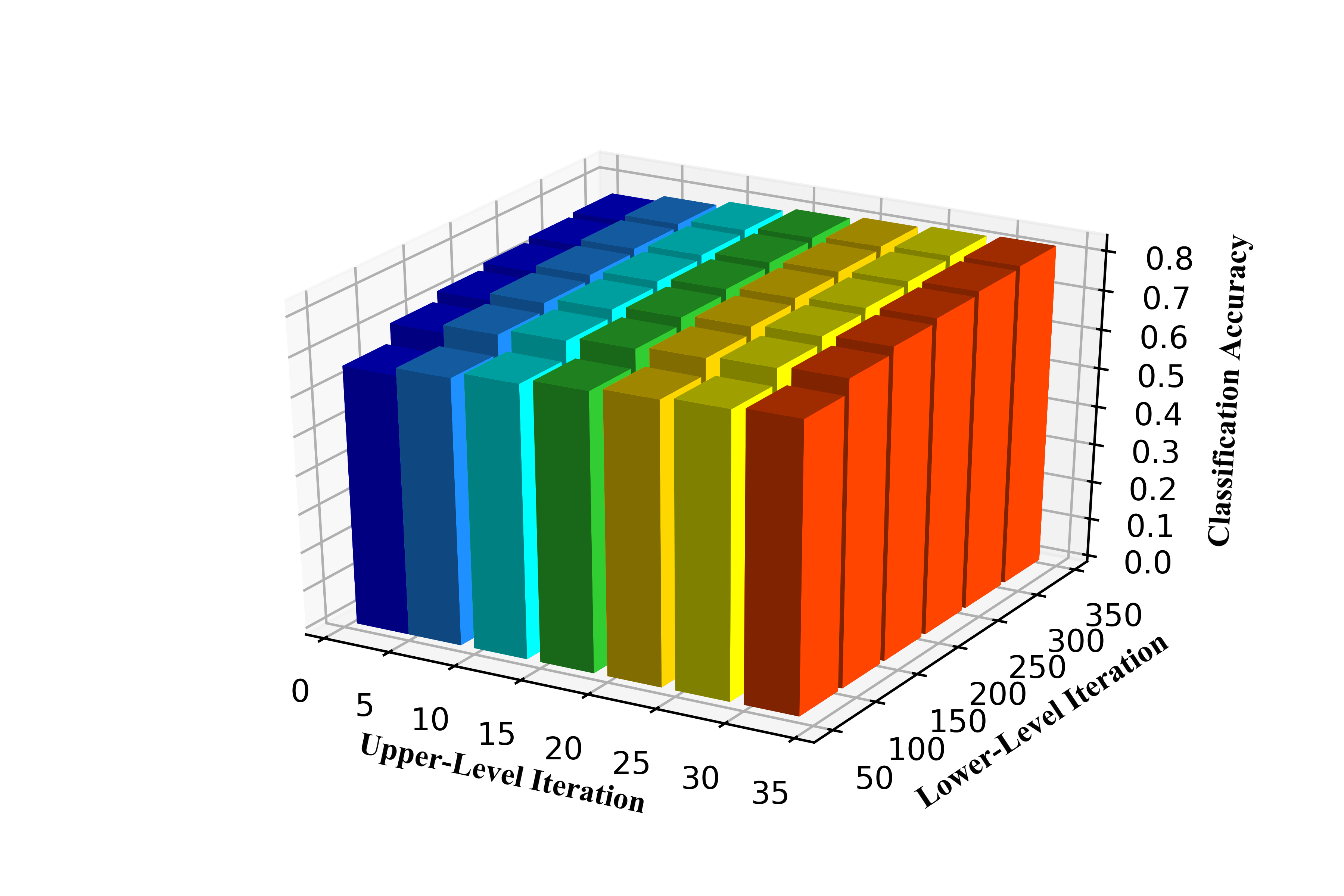}  
		\caption{The impact of the lower- and upper-level iterations.}
		\label{fig:iterations}
	\end{figure}
	
	We investigate the impact of the iteration number during the optimization dynamics on the quality of the solution. The result is plotted in Figure~\ref{fig:iterations}, where x-axis is the number of iterations for optimization of training loss, y-axis is the number of iterations for optimization of validation loss and z-axis is the classification accuracy. 
	
	From Figure ~\ref{fig:iterations}, we can see that, as the lower-level iterations increase, the better initialization we can obtain for the upper-level optimization problem, and the upper-level optimization only need a small number of iterations (such as 35 iterations) to achieve a quite good solution.
\end{document}